\documentclass{article}
\usepackage{lipsum}
\usepackage{amsfonts}
\usepackage{graphicx}
\usepackage{epstopdf}
\usepackage{algorithmic}
\usepackage{amsmath, amsthm, amssymb}
\usepackage{hyperref}
\usepackage{authblk}

\ifpdf
  \DeclareGraphicsExtensions{.eps,.pdf,.png,.jpg}
\else
  \DeclareGraphicsExtensions{.eps}
\fi

\usepackage{enumitem}
\setlist[enumerate]{leftmargin=.5in}
\setlist[itemize]{leftmargin=.5in}

\newtheorem{lemma}{Lemma}
\newtheorem{theorem}{Theorem}

\title{Modeling COVID-19 spread in the USA using metapopulation SIR models coupled with graph convolutional neural networks }
\author[a,b]{Petr Kisselev\footnote{peter.kisselev@gmail.com}}
\author[b]{Padmanabhan Seshaiyer\footnote{pseshaiy@gmu.edu}}

\affil[a]{\small Thomas Jefferson High School for Science \& Technology, Alexandria, VA}
\affil[b]{\small Department of Mathematical Sciences, George Mason University, Fairfax, VA}

\newcommand{\bit}{\begin{itemize}}
\newcommand{\eit}{\end{itemize}}
\newcommand{\ben}{\begin{enumerate}}
\newcommand{\een}{\end{enumerate}}
\newcommand{\beq}{\begin{equation}}
\newcommand{\eeq}{\end{equation}}
\newcommand{\beqa}{\begin{eqnarray*}}
\newcommand{\eeqa}{\end{eqnarray*}}
\newcommand{\bc}{\begin{center}}
\newcommand{\ec}{\end{center}}

\newcommand{\ds}{\displaystyle}
\newcommand{\bEps}{{\cal E}}
\newcommand{\cR}{{\cal R}}

\begin{document}

\maketitle
\section*{Abstract}
\hfill

Graph convolutional neural networks (GCNs) have shown tremendous promise in addressing data-intensive challenges in recent years. In particular, some attempts have been made to improve predictions of Susceptible-Infected-Recovered (SIR) models by incorporating human mobility between metapopulations and using graph approaches to estimate corresponding hyperparameters. Recently, researchers have found that a hybrid GCN-SIR approach outperformed existing methodologies when used on the data collected on a precinct level in Japan. In our work, we extend this approach to data collected from the continental US, adjusting for the differing mobility patterns and varying policy responses. We also develop the strategy for real-time continuous estimation of the reproduction number and study the accuracy of model predictions for the overall population as well as individual states. Strengths and limitations of the GCN-SIR approach are discussed as a potential candidate for modeling disease dynamics.

\section{Introduction}
\hfill

Compartment models are widely used in the modeling community to describe the spread of infectious diseases \cite{background}. Standard SIR model considers three compartments: $S(t)$ - the number of susceptible individuals, $I(t)$ - the number of infected and $R(t)$-the number of recovered or deceased at time $t$. Representing the infection rate parameter by $\beta$ and removal rate parameter by  $\gamma$, the following system of equations is derived \cite{stolerman}:
\beq
\left\{
\begin{array}{l}
\vspace{0.1in}
\ds \frac{d S(t)}{dt} = -\beta \frac{S(t) I(t)}{P}\\
\vspace{0.1in}
\ds \frac{d I(t)}{dt} = \beta \frac{S(t)I(t)}{P} - \gamma I(t)\\
\ds \frac{dR(t)}{dt} =  \gamma I(t)\
\end{array}
\right.
\label{standard}
\eeq
where $P=S(t)+I(t)+R(t)$ is the total population that is assumed to remain constant. 

There are many modifications of the basic SIR model that have been proposed in the literature \cite{review}. For example, the SEIR variation of the model includes another compartment for exposed individuals, while the SIRV variation incorporates a compartment vaccinated populations. It is also possible to account for individuals who end up succumbing to the disease and dying, as done in the SIRD variation and others like it. While these models are very intuitive and mathematically tractable, their predictive properties are highly dependent on the accuracy of the modeling parameters $\beta$, $\gamma$, and other parameters for the additional compartments . In fact, these basic parameters have been found to vary greatly between subpopulations for a variety of diseases including COVID-19 \cite{bertozzi, art1, art2, art3}. This realization has motivated several groups to develop so-called ``metapopulation SIR'' or ``SIR-network'' models which tackle this problem by splitting the overall population into a number of subpopulations and allowing for variable infection and recovery rate parameters across these newly created ``metapopulations''\cite{cao, stolerman}. These parameters may account for the change in mobilities between different compartments and differences in vaccination policies in different regions, among other conditions. 

The caveat of this approach is the increase in computational complexity, the need to estimate a larger set of parameters, and work with higher dimensional data. In this work, we explore the benefits of coupling the metapopulation(network) SIR model with the graph convolutional neural network (GCN) methodology which has enjoyed significant advances and popularity in recent years. One advantage of GCN parameter estimation compared to that of standard convolutional neural nets is its applicability to an arbitrary data structure as long as it may be represented by a graph. It is also better able to draw on geographical relationships. Several authors explored the GCN-SIR coupling, see review provided in \cite{review}.

Drawing motivation from the work of Cao et al \cite{cao}, we use GCNs to dynamically fit the parameters of the metapopulation SIR model using a given time series of data. Similar to \cite{cao}, we focus on making predictions on the spread of COVID-19 with the GCN framework given different ``horizons,'' and compare these forecasts with the standard SIR model. There are several distinctions in the approach presented in this work in comparison to the ``mepoGNN'' model of \cite{cao}. In particular: (1) the mepoGNN model was trained on Japanese precinct data, and our goal in this study is to model the spread of COVID-19 in the United States;  (2) in applying the original model to US data, we found the need to change several modeling assumptions including choosing a different form of the mobility parameter; (3) we analyzed the predictions of the model on specific subpopulations and estimated the overall reproduction number based on the metapopulation model.

\section{Graph Convolutional Neural Networks}
\hfill

Graph convolutional neural networks (GCNs) are an emerging technique that have shown promise in several areas. Here, we will be applying them to the prediction of daily infections of COVID data. Fundamentally, GCNs are very different from traditional neural networks as they act an expansion and modification of the core premise of how neural networks operate. Specifically, graph neural networks work to replace the structure of the data on which a traditional neural networks performs transformations, instead, applying transformations to data which represents a graph, utilizing relationships within this structure to enhance predictions and available context.

The traditional neural network, a model inspired by functioning of a biological brain, is a computational model designed to perform tasks such as classification, regression, and pattern recognition. It's structure consists of interconnected nodes, or ``neurons," which are then organized into layers: an input layer, one or more hidden layers, and an output layer. Each neuron takes in information from the previous layer and outputs a weighted sum of the inputs that involves parameters including weights and biases that are constantly learned. This is followed by the application of an activation function to introduce non-linearity and enable the model to learn complex patterns. Activation functions, such as the sigmoid or ReLU (Rectified Linear Unit), play a critical role in determining the network's ability to capture intricate relationships within data. Inputs to the network represent features of the problem being modeled, while the output corresponds to predictions or classifications. With the recent developments in computational power, neural networks have flourished, gaining significant popularity and enabling breakthroughs in many fields. An example of how neural networks have been used is for image recognition: the input is a picture containing either a cat or a dog and the model is trained to determine which of the two is depicted in the image. Figure \ref{fig:nn-structure} is a schematic representation of a standard neural network, illustrating the flow of information and transformations through its layers.

\begin{figure}
{\centering \includegraphics[width=\linewidth]{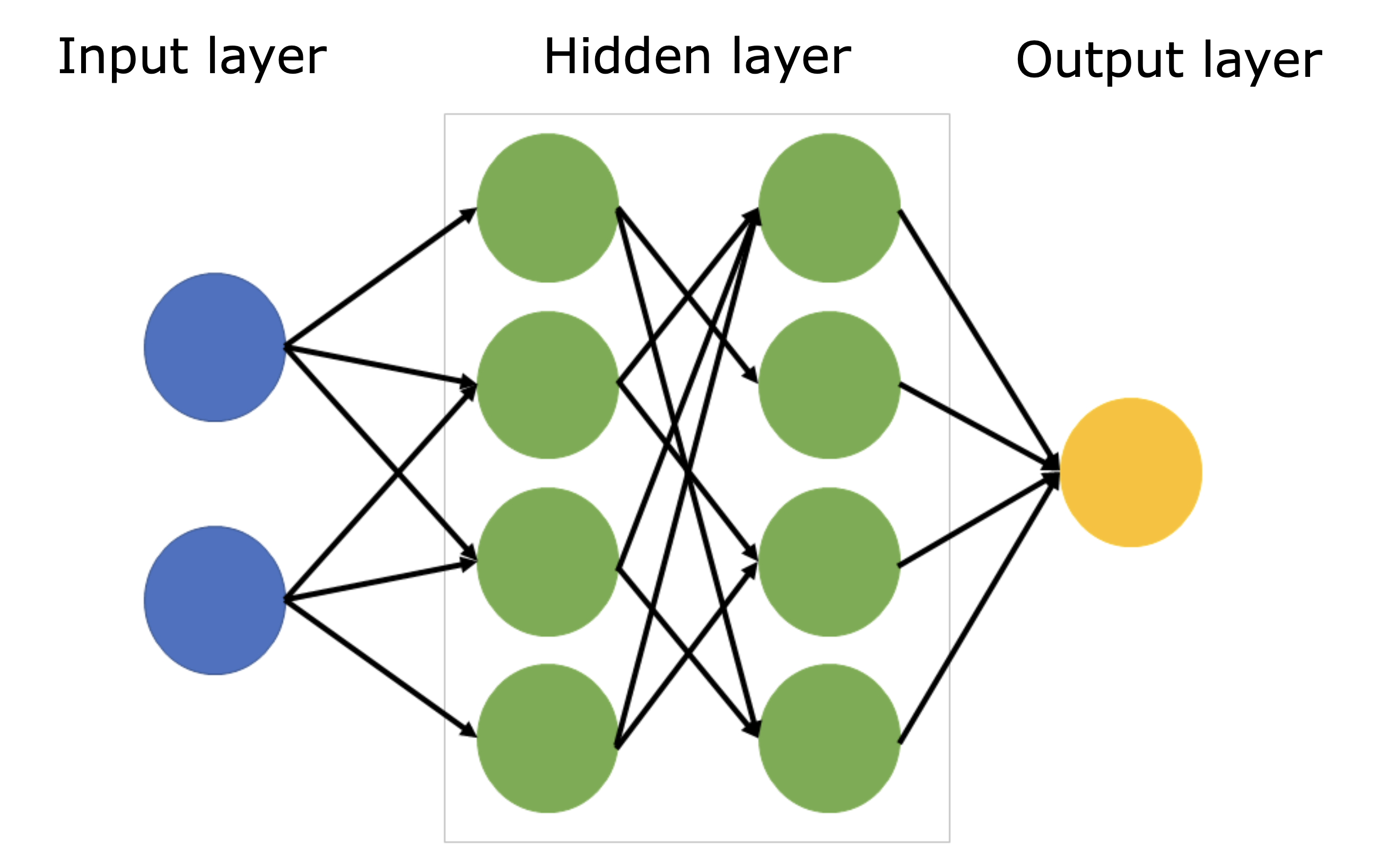}}
\label{fig:nn-structure}
\caption{Structure of a traditional neural network}
\end{figure}

The structure of the neural network that has been described can be enhanced with graph theory to model disease dynamics of metapopulations through SIR models (\ref{standard}). The purpose of coupling the metapopulation SIR model to a convolutional neural network is to achieve better accuracy in estimating hyperparameters by taking into account communication/mobility of sub-populations between different regions. This is accomplished by the mobility parameters assigned to the edges of a graph. More rigorously, graph neural network models may be defined as follows. Let $G$ be defined as the graph data such that $G=(V, \bEps)$. Where $V$ is defined to represent set of nodes comprised of $|V|=N$ nodes. Similarly, we let $\bEps$  be such that $\bEps \subseteq V \times V$. Here it will be used to store connection data between the nodes. The features may also be represented by the matrix $\textbf{X}=\{\textbf{x}_1, \textbf{x}_2,..., \textbf{x}_N\}^T \in \mathbb{R}^{N \times D}$, where the feature vector $\textbf{x}_i$ is associated with node $v_i$. Here, $D$ is used to denote dimension of the feature. By convention, we define an adjacency matrix for $G$ as $\textbf{A} \subseteq R^{N \times N}$, where $\textbf{A}_{ij} = 1$ for existing edges and $\textbf{A}_{ij} = 0$ otherwise. Figure \ref{fig:gnn-structure} is an illustration of embedding a graphical representation within the neural network framework that was described earlier. Specifically, the figure shows that the convolution is applied on a node-by-node basis, with the appropriate weights and biases.

\begin{figure}
{\centering \includegraphics[width=\linewidth]{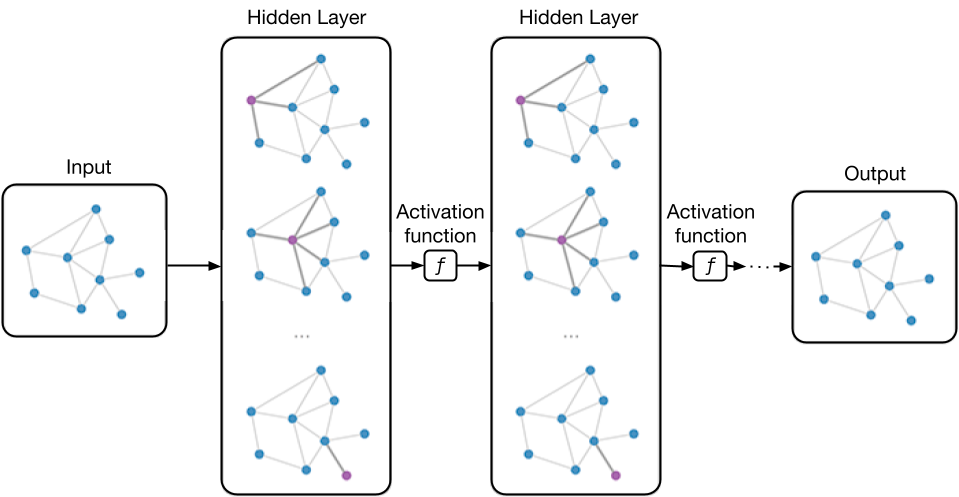}}
\label{fig:gnn-structure}
\caption{Structure of a GCN}
\end{figure}

In this paper we will be focusing on the subset of graph learning focused on node-level tasks. In other words, the GNN framework will be used to predict properties associated with individual nodes.  As with any other neural network model, it will be necessary to train it using a subset of nodes with known properties, or the training set. This training set of data will be denoted as  ${\cal V}_L$. The trained model will then be used to forecast the properties of unknown nodes from a separate testing set of data. The aforementioned training can be represented by the minimization of the following loss function:

\beq
{\cal L}(\textit{f}_\theta (G)) =\sum_{ v_i \in {\cal V}_L} \ell(\textit{f}_\theta (\textbf{X}, \textbf{A})_i ; y_i )
\label{eq: loss}
\eeq

Where $\theta$ is a vector containing the parameters of the model. The function $\textit{f}_\theta(\textbf{X}, \textbf{A})$ is designated to forecast property values for each node, where $y_i$ is defined to represent the true state of the node $v_i$. The difference between the predicted and true properties ($\textit{f}_\theta(\cdot, \cdot)_i$ and $y_i$ respectively) is quantified using a loss function $\ell (\cdot, \cdot)$. Examples of loss functions that can be used include RMSE (Root Mean Square Error), MAE(Mean Absolute Error), smooth $L_1$ loss, and others. 

\section{Model description}
\hfill\break
Network SIR models for a total of $M$ subpopulations typically have the form \cite{stolerman}:

\beq
\left\{
\begin{array}{ll}
\vspace{.1in}
\ds \frac{dS_n}{dt} &=\ds  - S_n \sum^M_{m=1}{\beta_{mn}I_m}\\
\vspace{.1in}
\ds \frac{dI_n}{dt} &= \ds S_n \sum^M_{m=1}{\beta_{mn}I_m} - \gamma_n I_n\\
\ds \frac{dR_n}{dt} &= \gamma_n I_n
\end{array}
\right.
\label{model0}
\eeq
where $\beta_{mn}$ are the corresponding interaction parameters accounting for the movements between subpopulations. Cao et al \cite{cao} chose the form $\beta_{mn}=\beta_n\left(\frac{h_{mn}}{P_m}+\frac{h_{nm}}{P_n}\right)$, where $h_{mn}$ modeled mobility between regions $m$ and $n$, and $P_m,P_n$ represent the populations of the regions, respectively. We denote the total population of the country by $\ds P=\sum^M_{m=1}P_m$. This leads to the following form of the metapopulation SIR model:
\beq
\left\{
\begin{array}{ll}
\vspace{.1in}
\ds \frac{dS_n}{dt} &=\ds  - S_n \beta_n\sum^M_{m=1}{\alpha_{mn}I_m}\\
\vspace{.1in}
\ds \frac{dI_n}{dt} &= \ds S_n \beta_n \sum^M_{m=1}{\alpha_{mn}I_m} - \gamma_n I_n\\
\ds \frac{dR_n}{dt} &= \gamma_n I_n
\end{array}
\right.
\label{model}
\eeq

where $\alpha_{mn}=\frac{h_{mn}}{P_m}+\frac{h_{nm}}{P_n}$ are the interaction coefficients modeling mobility. In their work, they argued that following form of mobility was best suited for this task: $h_{mn}=\alpha \frac{P_n P_m}{(dist_{mn})^d +\epsilon} $, where $\alpha, \epsilon$ and $d$ are the training hyperparameters and $dist_{mn}$ is the distance between the regions. 

By representing the SIR model via a graph with mobilities $\alpha_{mn}$ assigned to the edges and keeping $\gamma_n$ to represent recovery/immunity rate, a graph convolutional neural network was trained using the somewhat complex architecture that has been claimed to be the first hybrid model that couples the metapopulation SIR model with spatiotemporal graph neural networks. The code shared by the authors has been used as a basis for our investigation, where we implemented the GCN structure as shown in Figure \ref{fig:architecture}. This architecture consists of three main sections: the graph learning module, the metapopulation SIR module, and the spatio-temporal module. For our model we chose to use the adaptive version of the model proposed by Cao et al, where the mobilities are initialized statically from an estimation based on population and distance, and then learned.  The spatio-temporal module is comprised of a combination of spatio-temporal (ST) layers. Each layer is created through combining a graph convolutional neural network layer with a gated temporal convolutional network layer. The results from the spatio temporal module are then passed through two fully connected layers with a  ReLU (Rectified Linear) activation function and a Sigmoid activation function respectively, producing the predicted $\beta$ and $\gamma$ parameters \cite{literate}. These predicted parameters are fed into the final, metapopulation SIR module (see Figure \ref{fig:architecture}) which is detailed above, which produces the final daily infection prediction from the model.

\begin{figure}
{\centering \includegraphics[width=\linewidth]{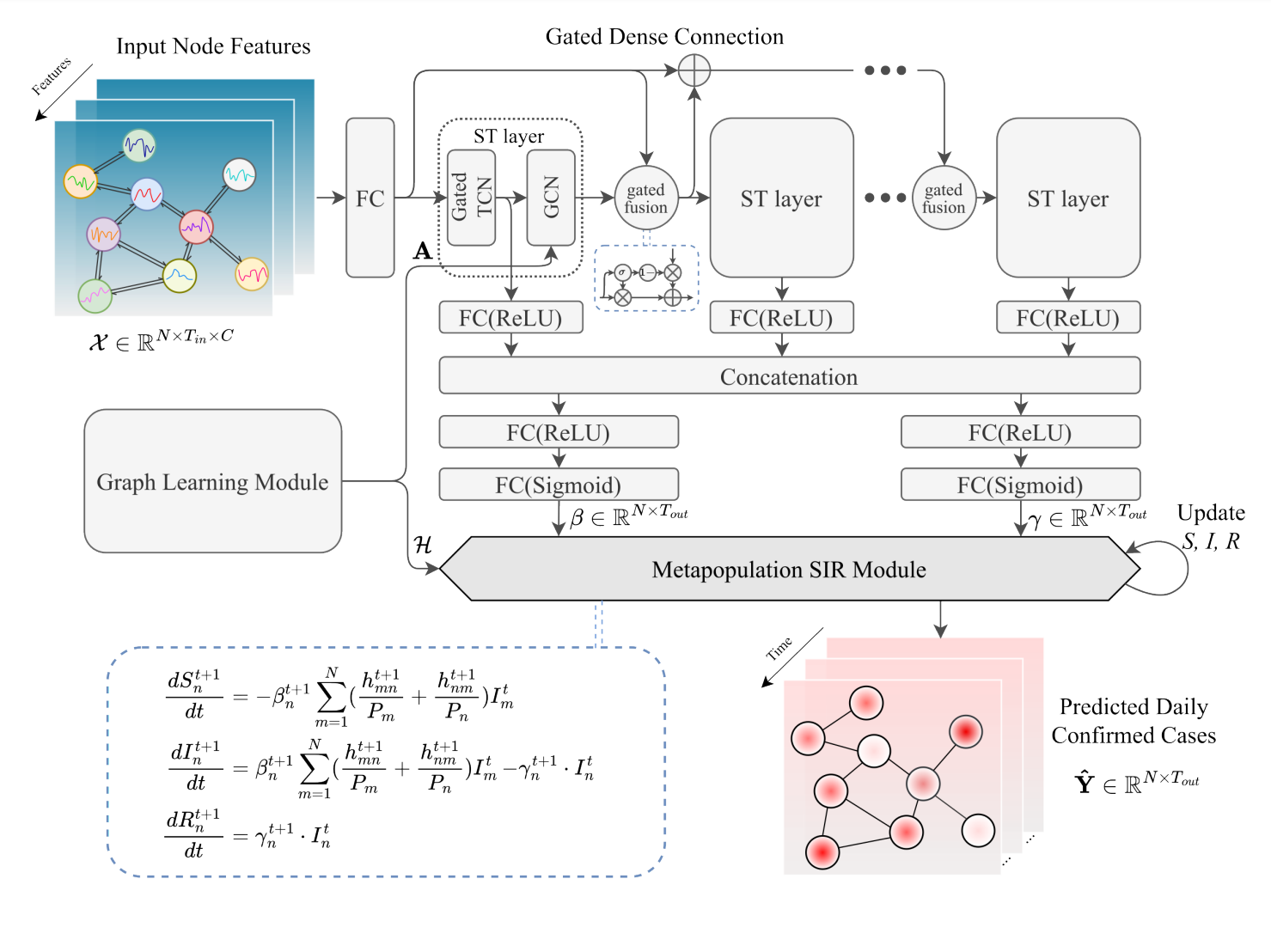}}
\label{fig:architecture}
\caption{Model architecture used in \cite{cao}}
\end{figure}

\section{Extending the model to US data}
\hfill

In extending the original model to US data, we faced several challenges.

First, there was a lack of an easily accessible source for recovery data. In order to gather the data necessary it was necessary to take several key steps. We started by sourcing data on daily infections from a dataset downstream from official data by the Johns Hopkins Center for Systems Science and Engineering (CSSE). This provided confirmed infection data on a county level which was binned up to the state level to alleviate computational complexity concerns. Additionally, the date standard format was reinterpreted as an integer day offset from the first day found in the dataset. We also made the decision to exclude US territories, the District of Columbia, Alaska, and Hawaii from the training and testing of our model as this would introduce additional complexity into the geospatial relationships without benefitting the predictions made significantly. Several supplementary datasets were additionally used, such as data on the physical locations of state centers and US state populations. We were, however, unsuccessful in obtaining a dataset that could sufficiently detail recovery data in the United States. This was a challenge as the model necessitated such data as an input, data which was available in Japan but not for the US. To overcome this issue, we generated recovery parameter by numerically solving System \ref{model} using an Euler approximation from an ad-hoc $\gamma$ value and the known real-world infection data which we sourced. We believe that this solution is effective because while the recovery data is somewhat approximate, the model's performance in predicting infections is still compared against the ground truth. 

Second, compared to Japan, United States complex interaction between the regions in the United States is more complex, so we had to account for effect of state-level policies on mobility estimation

The mobility value approximation has been improved with an additional term in the formula to account for flight travel:
\beq
h_{mn}=\alpha \frac{P_n P_m}{(dist_{mn})^d +\epsilon} + \beta \max(P_n,P_m) (1-\delta_{mn})
\label{mobility}
\eeq
Namely, the last term allows to have significant mobility between densely populated states even if the distance between them is large. Multiplying by the Kronecker delta function makes sure this term collapses to zero in the simple SIR model case when $M=1$.  

The following consistency analysis has been carried out. As a result of the model architecture, the mobility values are not normalized and so it is necessary to balance them out using the following:
\begin{lemma}
Metapopulation model \eqref{model} is consistent with the standard SIR model if and only if $2\alpha P^2 = \epsilon$.
\label{lemma1}
\end{lemma}
\begin{proof}
Taking the limiting case of $M=1$ subpopulation and denoting $h_{mn}=h_{nm}=h$, $\beta_n=\beta, \forall n=1,\ldots, M$, we obtain $\alpha_{nn}=\ds \frac{2h}{P}$, where $P=P_n=P_m$ for all $n,m$. It is clear that since $h_{mn}=\ds \alpha \frac{P_n P_m}{(dist_{mn})^d +\epsilon} $, $h = \ds \alpha \frac{P^2}{\epsilon}$, so $\ds \alpha_{nn}=\frac{2\alpha P}{\epsilon}$. Hence 
$ S_n \beta_n\sum^M_{m=1}{\alpha_{mn}I_m} = \ds \beta \frac{2\alpha P}{\epsilon} S_n I_m$ which is equal to $\beta SI/P$ under the condition that $2\alpha P^2 = \epsilon$
\end{proof}
This allows to reduce the number of free parameters to $\alpha$ and $d$, simplifying the training of the mobility parameters. Optimized values of these parameters were chosen as follows:
\begin{equation*}
\begin{array}{l}
\alpha:  1.12 \times 10^{-6}\\
d:  1.73 \mbox{ (distance decay factor)}\\
\end{array}
\label{hyper-mob}
\end{equation*}

Parameter $\epsilon$ was fixed in accordance with Lemma \ref{lemma1}.  The optimized GCN hyperparameter values and training details are provided below:
\begin{equation*}
\begin{array}{ll}
 \mbox{ Learning Rate}:  2.5 \times 10^{-5}\\
 \mbox{ Optimizer}:  \mbox {Adam}\\
 \mbox{ Loss Function}:  \mbox {MAE (Mean Absolute Error)}\\
 \mbox{ Epochs}:  319 \\
\end{array}
\label{hyper-gcn}
\end{equation*}

It must be noted that some run-to-run variance is expected in this model due to the random nature of weight initialization in the GCN training. Additionally, hardware differences may also slightly change results as using the GPU, CPU, or a dedicated accelerator will lead to having minor differences in the driver and PyTorch backend implementations.

\section{Real-time tracking of the reproduction number}
\hfill

One of the critical considerations that is important to keep in mind when modeling COVID-19, as well as other infectious diseases, is the ability of the model to predict its spread. The threshold parameter  ${\cR}_0$, such that the disease free equilibrium (DFE) is asymptotically stable for ${\cal R}_0<1$ and unstable otherwise, is called the basic reproduction number. A more granular parameter accounting for the changes in population susceptibility, is the so-called effective reproduction number, denoted as ${\cR}_t={\cR}_0\frac{S_t}{N}$. Both measures are important tools for the mathematical validation of epidemiological models, as well as for practical considerations. Challenges and misconceptions in estimating these metrics are well documented \cite{gostic, failure}. 

Local parameters of the evolving epidemic change based on mobility patterns, population density and policy measures, the complexity of which creates significant difficulties for decision-making. The need for accurate continuous real-time prediction of the reproduction numbers in light of this variability has long been recognized and documented in the literature \cite{shapiro}. Some of the proposed real-time estimation methods include the adaptive SIR methodolgy (ASIR, \cite{shapiro}), where ${\cR}_0$ is based on a sliding time window approach, and the introduction of an ``effective contact rate'' to capture incidence dynamics over a given network \cite{Romanescu}. We argue that the graph neural network approach chosen in this work has a natural capability to capture the evolution of modeling parameters in real-time, and hence it may provide an opportunity to improve upon prior ${\cR}_0$ predictions. 

As noted in \cite{Chavez}, there is a natural connection between ${\cR}_0$ and ${\cR}_t$ when it comes to studying SIR population models. Namely, as we look at the equation for the infected population in the standard SIR model, 
$$
\frac{dI}{dt} = \beta S \frac{I}{N} - \gamma  I = \gamma I (\frac{\beta}{\gamma} \frac{S}{N} -1)=\gamma  ({\cR}_0 \frac{S}{N}-1)I = \gamma (\cR_t-1)I.
$$
The role of the bifurcation parameter ${\cR}_t$ is clear. It separates stable behavior of the disease-free equilibrium $I^*=0$, for which $\frac{dI}{dt}<1$ (for $\cR_t<1$) from the unstable and possibly endemic equilibrium when $\cR_t>1$.

For the metapopulation SIR model considered in this paper, we can take a similar approach, following the framework described in \cite{Driessche}. Namely, 
\begin{theorem}
Basic reproduction number for model \ref{model} is given by
$\cR_0 = \rho(D A)$, where 
$A = \{A_{nm}\}=\{\alpha_{nm}\}$ is the mobility matrix and 
$D = \mbox{diag}(\frac{ \beta_1}{\gamma_1}, \ldots, \frac{ \beta_m}{\gamma_m})$  is the scaling matrix.
\end{theorem}
\begin{proof}
It is easy to see that the Jacobian of this model, linearized around the Disease Free Equilbrium (DFE) $x^*$, can be represented as $DF(x^*)= F-V$, where $F_{nm}=\beta_n P_n \alpha_{nm}$ and $V = \mbox{diag}(\gamma_1, \ldots, \gamma_m)$. Here, ``$F - V$" refers to the \textit{next generation matrix}, where ``$F$" represents the inflow and ``$V$" represents the outflow, and the basic reproduction number is calculated as the maximum eigenvalue of the matrix ``$FV^{-1}$". We used the fact that the DFE is represented by $I_n^*=0,S^*_n=P_n$ for this model. As shown in \cite{Driessche}, the DFE is stable when $\rho(FV^{-1})<1$ under certain conditions on $F$ and $V$ that can be shown to hold in this case. Henceforth we arrive at the conclusion that for this model:
\beq
\begin{array}{l}
\vspace{.1in}
\cR_0 = \rho(D A), \mbox{ where } \\
\vspace{.1in}
A = \{A_{nm}\}=\{\alpha_{nm}\} \mbox{ is the mobility matrix, } \\
D = \mbox{diag}(\frac{ P_1\beta_1}{\gamma_1}, \ldots, \frac{ P_m \beta_m}{\gamma_m}) \mbox { is the scaling matrix}
\end{array}
\label{eqn: R0}
\eeq
which proves the result of this theorem.
\end{proof}
From Lemma \ref{lemma1}, we know that in the limiting case of $M=1$ we have to satisfy 

\[\ds \alpha_{nn}=\frac{2\alpha P}{\epsilon}=\displaystyle\frac{1}{P}, \] 

 so since $\ds D =P \frac{\beta}{\gamma}$ the reproduction number of the metapopulation model $\cR_0$ in this case converts to the well known result $\ds \cR_0= \frac{\beta}{\gamma}$.

This result provides a method for continuous evaluation of $\cR_0$ based on the evolving set of infection parameters estimated by the neural network. As the neural net learns and adjusts the underlying mobility and recovery rates, we can use this estimation to recover real-time reproduction number.

\section{Numerical results} 
\hfill

Results of our numerical experiments produced by applying the modified mepoGNN model to US data are presented next. In particular, US state center-center distances and state population data was obtained from Kaggle \cite{kaggle}. Confirmed US COVID-19 cases were collected from the Github repository maintained by The Center for Systems Science and Engineering at Johns Hopkins University \cite{JHU}.

\begin{figure}
{\centering \includegraphics[width=0.5\linewidth]{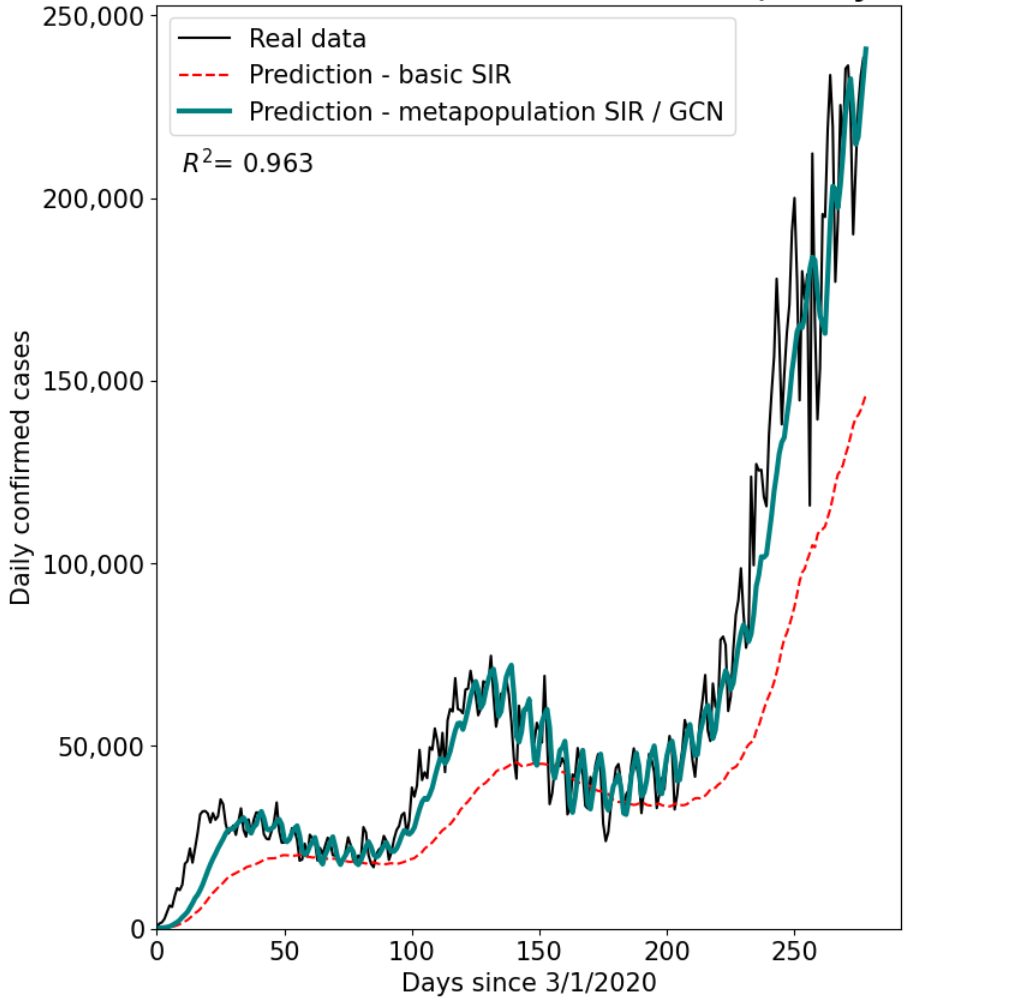}} 
\hfill
{\centering \includegraphics[width=0.5\linewidth]{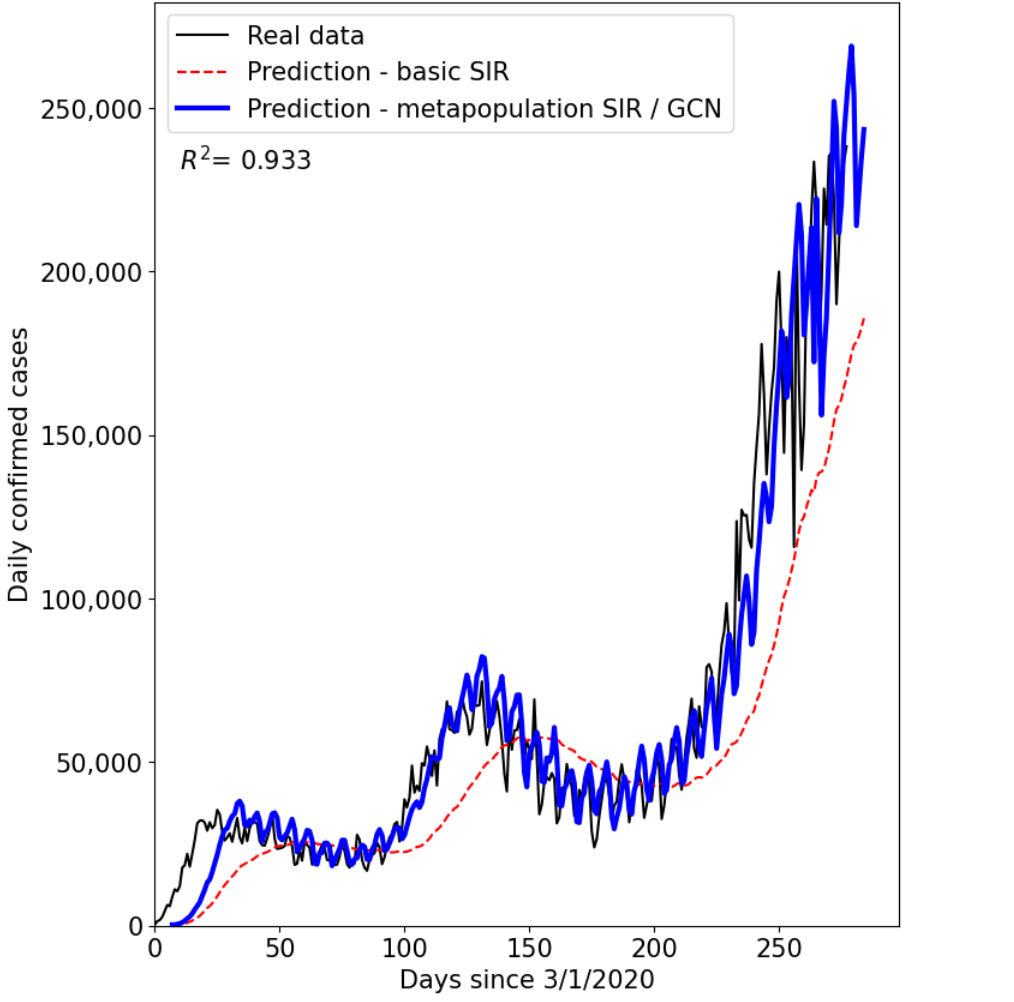}}
\label{fig: mepo1}
\caption{Metapopulation model prediction for US data based on real COVID-19 data, compared against the standard SIR model. Left panel: SIR-GCN Predictions on US COVID-19 data, 1 day horizon; Right Panel: SIR-GCN Predictions on US COVID-19 data, 7 day horizon}
\end{figure}

Figure \ref{fig: mepo1} shows predictions from the trained model based on a 1-day (left) and 7-day (right)  horizon, respectively. The GCN-SIR model is juxtaposed with the standard SIR model trained on the same dataset. It can be seen that by taking into account variability between regions, the model improves upon the prediction provided by traditional SIR approaches. In addition, it takes advantage of the neural network's learning capabilities to effectively train model parameters. 

In Figure \ref{fig: states} we look at the accuracy of the model predictions per state, choosing Virginia, New York, California, Ohio, Rhode Island and North Dakota as a sample containing large and small subpopulations. What we see is a strong model prediction for the densely populated states (New York, Virginia, California, Ohio) and a poor prediction for the less populous states (North Dakota, Rhode Island). 

\begin{figure}
{\centering \includegraphics[width=0.5\linewidth]{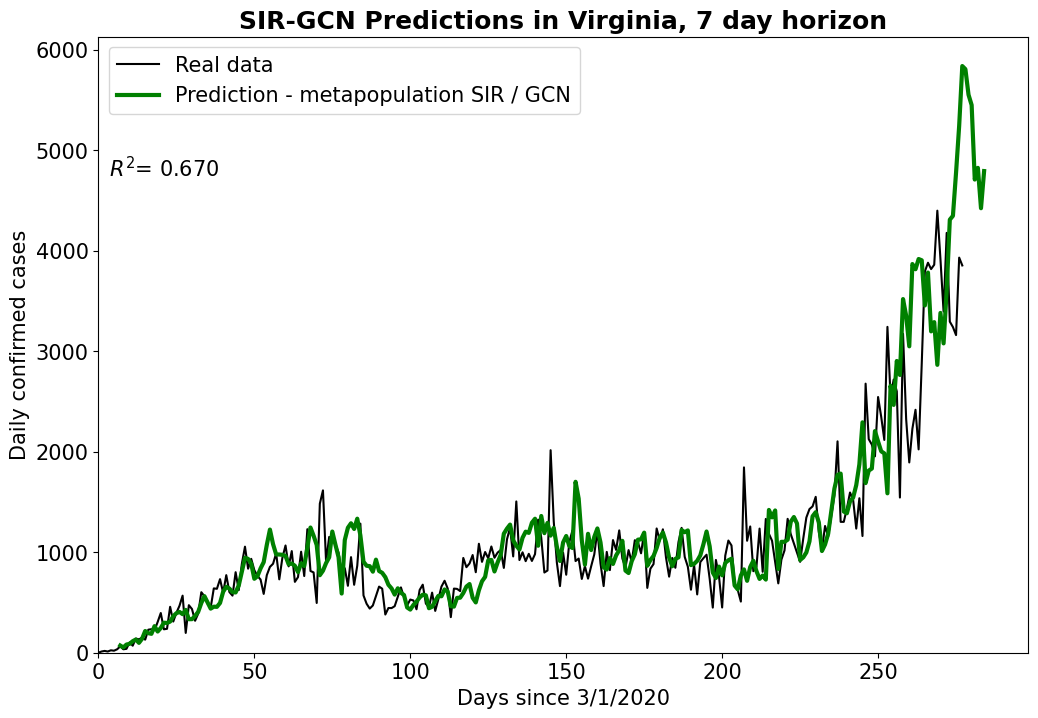}} 
\hfill
{\centering \includegraphics[width=0.5\linewidth]{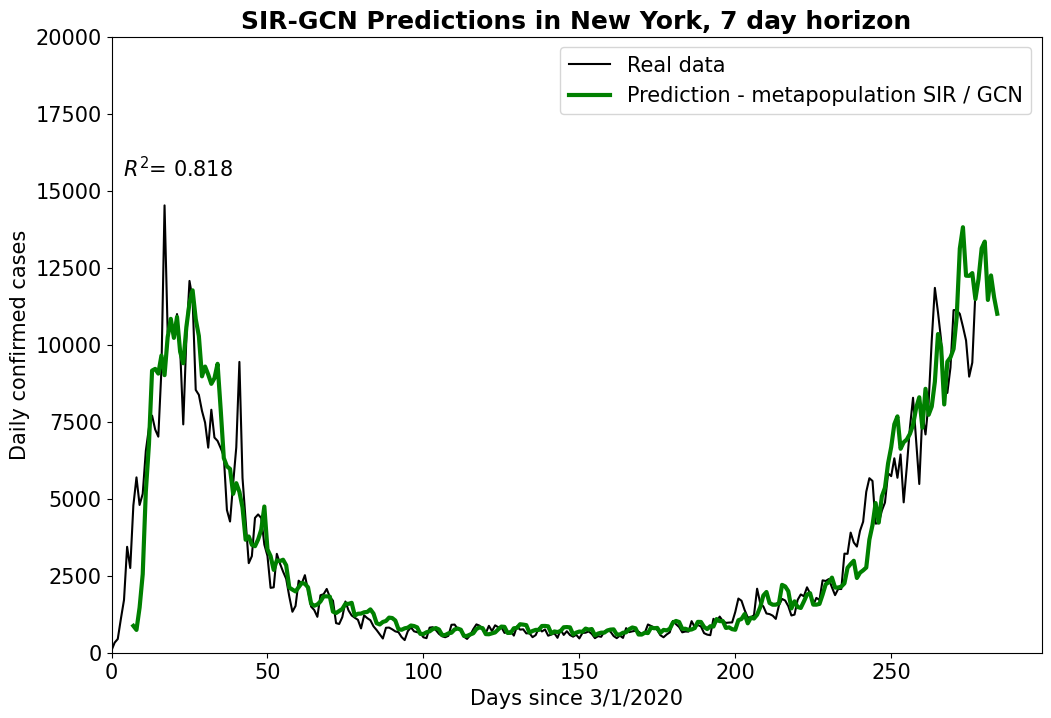}}\\
{\centering \includegraphics[width=0.5\linewidth]{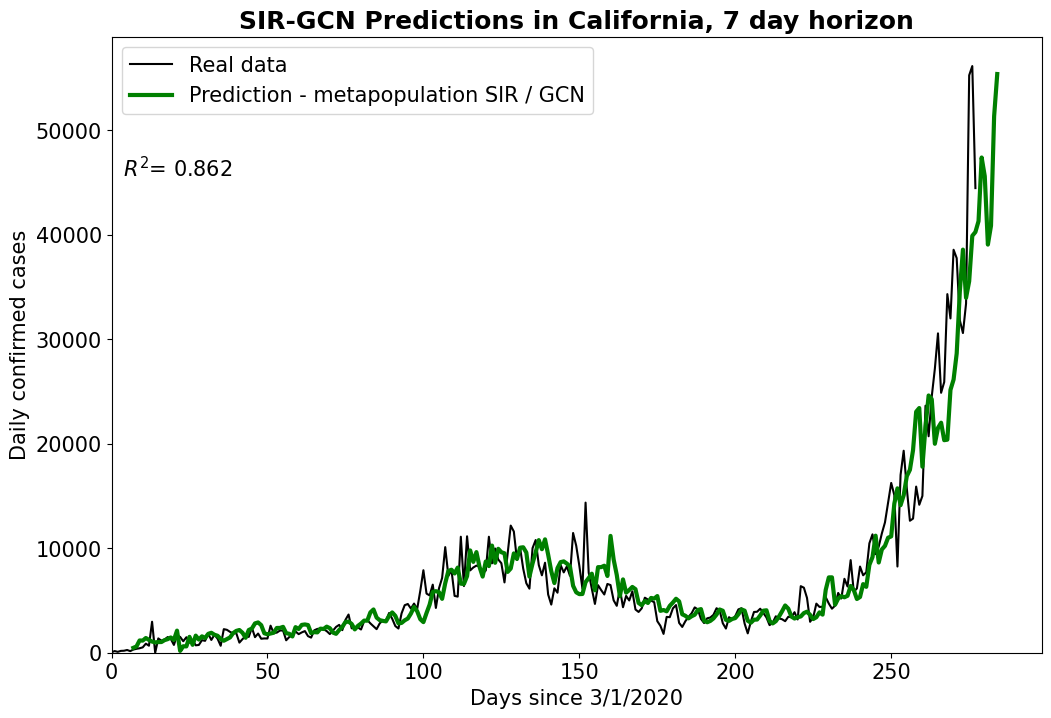}} 
\hfill
{\centering \includegraphics[width=0.5\linewidth]{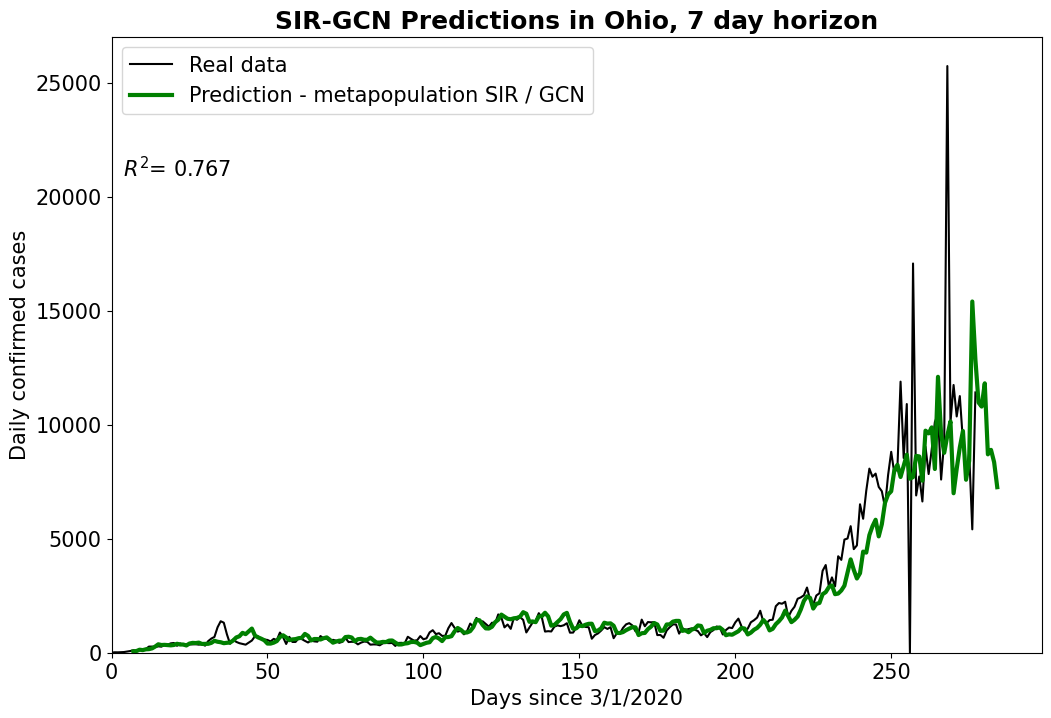}}\\
{\centering \includegraphics[width=0.5\linewidth]{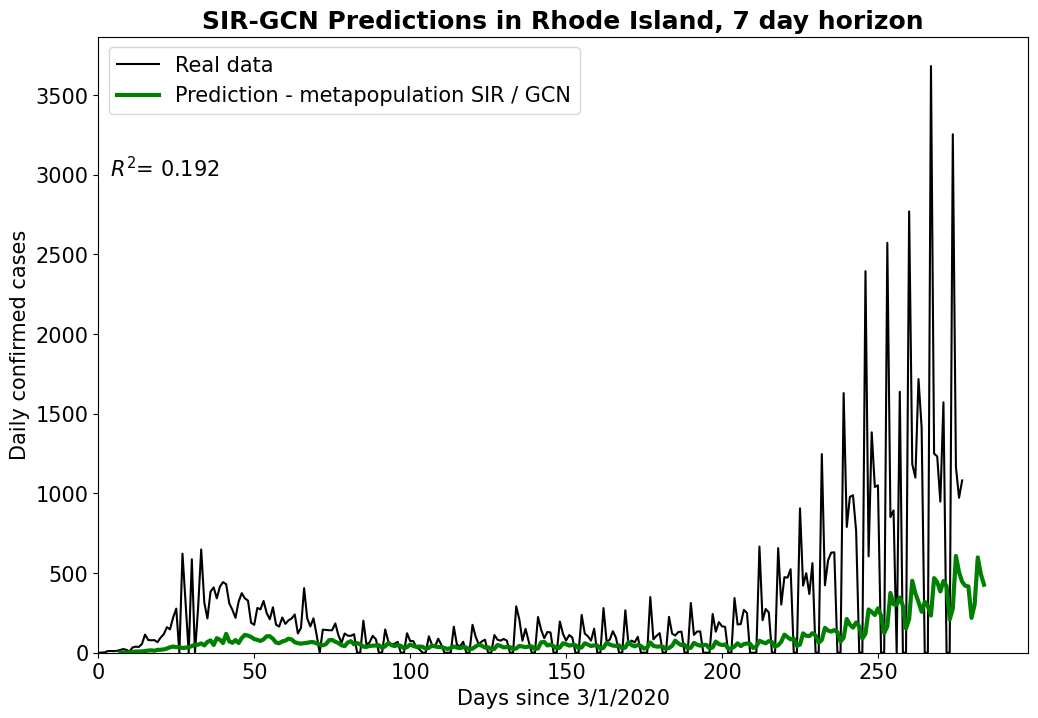}} 
\hfill
{\centering \includegraphics[width=0.5\linewidth]{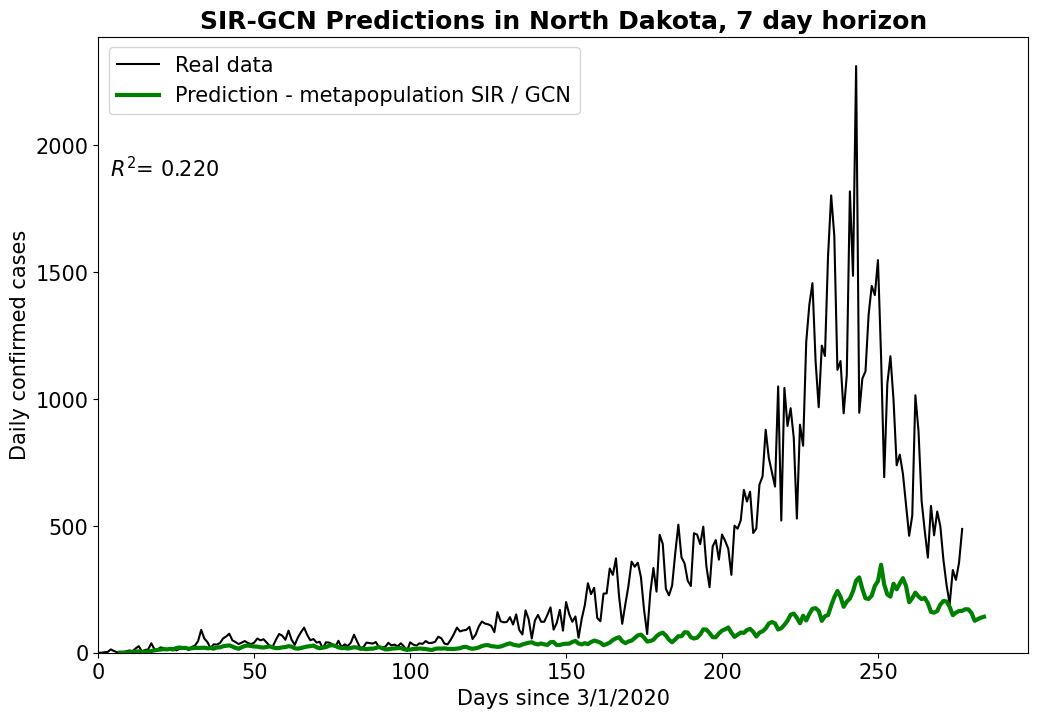}}
\label{fig: states}
\caption{Metapopulation model predictions for six US states.}
\end{figure}
To test this hypothesis, we plotted the correlation between the $R^2$ measure of fit and the corresponding state size in Figure \ref{fig: correlation}. To account for the large variation in state populations, the populations are log-scaled. The graph clearly shows moderate correlation, confirming that large-size subpopulations enjoy a more accurate prediction by the GCN-SIR model, which is to be expected given that the size plays a critical part in the optimization algorithm used in training the GCN. It is also clear that a majority of the state-level predictions have an $R^2>0.6$, which indicates reasonable performance overall.

\begin{figure}[h]
\centering{\includegraphics[width=0.7\linewidth]{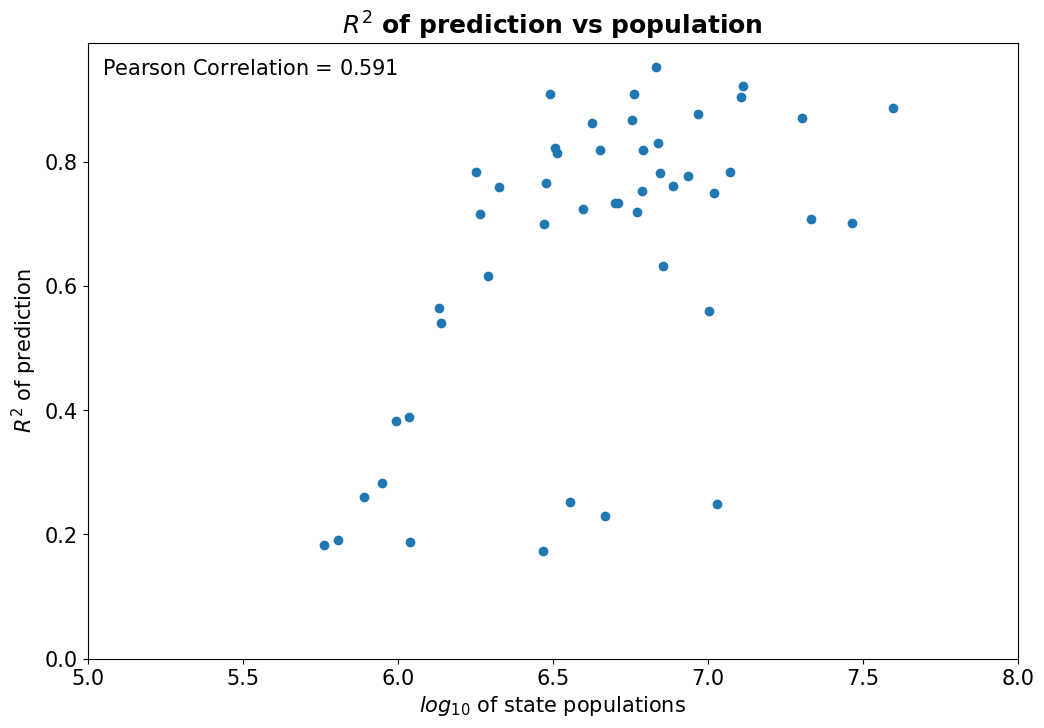}}
\caption{Correlation between the accuracy of fit for the metatpopulation SIR model and the size of the state for 48 contiguous Unites states.}
\label{fig: correlation}
\end{figure}

Next, we performed numerical experiments to continuously estimate the reproduction number (${\cR}_0$) of the entire metapopulation model using the estimate derived earlier in \eqref{eqn: R0}. The numerical results of this estimation, compared to the standard SIR reproduction number calculation, are given in Figure~\ref{fig: R0}. We can see the evolution of the ${\cR}_0$ value over the course of the pandemic, roughly capturing the ups and downs of the infection represented in Figure ~\ref{fig: mepo1}. The higher frequency oscillations visible in the graphs are due to the day-by-day variations in the neural net predictions and the real-world fact that people tend to travel more on certain days of the week than others.
\begin{figure}[h]
\centering {\includegraphics[width=0.7\linewidth]{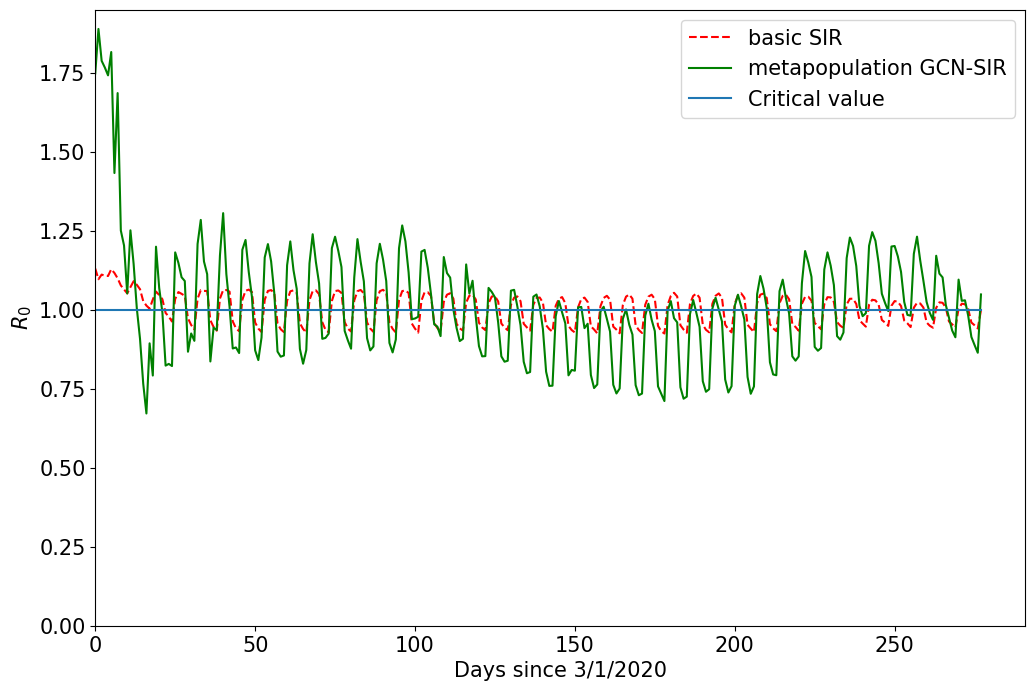}}
\caption{${\cR}_0$ number estimation using the GCN-SIR model.}
\label{fig: R0}
\end{figure}
It must be noted that while the overall ${\cR}_0$ values obtained by this approach were comparable to those available in the literature, state-level predictions were far less accurate. It indicates that a more granular county-based approach might be necessary to resolve state-level estimations. 

\newpage 

\section{Discussion and future work}
 \hfill\break
In this work we successfully adapted the hybrid GCN-SIR metapopulation model to predict the evolution of COVID-19 in the 48 continental states of the United States of America. In order to do so, we changed the formulations of the mobility parameters and derived the reproduction number formulation compatible with the standard SIR model. This allowed to streamline the process for training the hyperparameters to obtain a more robust implementation. 

Upon implementing these changes, we were able to obtain a high accuracy predictions for both 7-day and 1-day horizons for the entire United States. We noticed that individual state prediction accuracy was correlated with the state population, with densely populated states enjoying a better fit. 

Based on the neural network based approach to learn the infection rates in real time, we developed an alternative to the adaptive SIR method for estimating the reproduction numbers. Applying this approach to the entire US population, a reasonable prediction has been obtained, giving reason to believe that further improvements may yield an even better predictive capability that would be of significant interest to policy makers and medical practitioners. 

Overall, based on the results presented in this work, GCN-SIR metapopulation model seems to have a high potential for predicting  improving predictions of the spread of infectious diseases based on sufficient amount of training data. To our knowledge, this is the first application of this type of a GCN-SIR coupling to real COVID-19 data collected within the USA. 

While these preliminary results are encouraging, we believe that additional work needs to be performed to validate the model on other types of data. High correlation of the $R^2$ fitting parameter with the size of the subpopulations indicates that further improvements may be made to the choice of the mobility formulation, including learning mobility matrices in real time. Additional work could include building a better mobility estimation based on a more granular county-level data. All of our current attempts at a more granular model so far have run into issues with handling the sheer size of the model.

Future work will include deriving more accurate state-level reproduction number estimation and improvement of parameter estimation procedures. The possibility of additionally including the effect of local policy changes into the model is also one that we will consider in the future.\\

\noindent{\textbf{Acknowledgements}}

This work is partially supported by the National Science Foundation grant DMS-2230117.

\bibliographystyle{unsrt}
\bibliography{draft_GCN_arXiv}

\end{document}